\documentclass[11pt]{article}
\usepackage{amsmath,amsfonts,amssymb,amsthm,  mathrsfs,mathtools}
\usepackage[margin=1in]{geometry}
\usepackage{hyperref}
\hypersetup{
  colorlinks,
  citecolor=violet,
  linkcolor=blue,
  urlcolor=blue}
\usepackage{dsfont}
\usepackage{algorithm}
\usepackage{algorithmic}
\usepackage{subcaption}
\usepackage{xfrac}
\usepackage{enumitem}
\usepackage{graphicx}
\usepackage[dvipsnames]{xcolor}
\usepackage[square,numbers]{natbib}
\usepackage{algorithm}
\usepackage{natbib}

\usepackage{authblk}
\usepackage{thm-restate}

\usepackage{wrapfig}
\usepackage{booktabs}
\usepackage{multirow}

\DeclareMathOperator{\Points}{\mathcal{C}}

\DeclareMathOperator{\Colors}{\mathcal{H}}

\newcommand{\OR}{\textbf{OR}}
\newcommand{\ML}{\textbf{ML}}

\newcommand{\DS}{\textbf{DS}}
\newcommand{\GF}{\textbf{CM}}
\newcommand{\SF}{\textbf{SF}}

\newcommand{\EQ}{\textbf{EQ}}
\newcommand{\WC}{\textbf{WC}}

\DeclareMathOperator{\POF}{\mathrm{PoF}}

\newcommand{\anomeps}{\mathnormal{c}}

\author[1,2]{John Dickerson}
\author[3]{Seyed A. Esmaeili}
\author[4]{Jamie Morgenstern}
\author[4]{Claire Jie Zhang}
\affil[1]{University of Maryland, College Park}
\affil[2]{Arthur}
\affil[3]{University of Chicago}
\affil[4]{University of Washington}

\begin{document}

\title{Fair Clustering: Critique, Caveats, and Future Directions}

\date{}  
\maketitle

\begin{abstract} 
Clustering is a fundamental problem in machine learning and operations research. Therefore, given the fact that fairness considerations have become of paramount importance in algorithm design, fairness in clustering has received significant attention from the research community. The literature on fair clustering has resulted in a collection of interesting fairness notions and elaborate algorithms. In this paper, we take a critical view of fair clustering, identifying a collection of ignored issues such as the lack of a clear utility characterization and the difficulty in accounting for the downstream effects of a fair clustering algorithm in machine learning settings. In some cases, we demonstrate examples where the application of a fair clustering algorithm can have significant negative impacts on social welfare. We end by identifying a collection of steps that would lead towards more impactful research in fair clustering. 
\end{abstract}
\section{Introduction}
Machine learning and algorithmic decision making are seeing widespread use in society, affecting the welfare of individuals in numerous and impactful ways from loan approval and hiring, to recidivism prediction and kidney exchange \cite{tejaswini2020accurate,arun2016loan,purohit2019hiring,berk2021fairness,liu2011comparison,berk2013statistical,awasthi2009online,aziz2021optimal,mcelfresh2019scalable,raghavan2020mitigating}. This has pushed fairness considerations to the forefront and instigated a large body of work in algorithmic fairness. Unsurprisingly, clustering being a classical problem in operations research and arguably the most fundamental problem in unsupervised learning has received significant attention from the research community that has resulted in tens of papers (see for an incomplete list \cite{fc_tutorial,chhabra2021overview,chierichetti2017fair,bercea2018cost,bera2019fair,abbasi2020fair,ghadiri2021socially,ahmadi2020fair_correlation,ahmadian2019clustering,anderson2020distributional,brubach2020pairwise,brubach2021fairness,backurs2019scalable,schmidt2018fair,esmaeili2020probabilistic}). Because of the impact of the problem and its widespread use, the emergent field of \emph{fair clustering} has the potential of being quite impactful. The field has generated interesting and elaborate notions of fairness and novel algorithms for solving them. Despite this progress, a collection of issues have been neglected. In this paper, we highlight and expand on a collection of important overlooked issues in fair clustering. We demonstrate that many of these issues are consequential for real life applications of fair clustering including cases where harm can possibly be caused because of fair clustering whereas an agnostic (fairness unaware) clustering would not result in such harm. 

Algorithmic fairness is still a developing field and it is therefore not difficult to point out shortcomings. Among the existing critiques, \citet{selbst2019fairness} discuss possible reasons for the failure of fair machine learning in large sociotechnical systems. More specifically, fair machine learning research is criticized as using abstractions to create homogeneous learning tasks taken out of their original contexts where researchers then provide standalone and portable solutions which are often misused. Further, \citet{holstein2019improving} highlights the disconnect between the challenges faced by practitioners and the support provided by fair machine learning researchers. Other problems pointed out that exist in almost all paradigms include ignoring long term consequences on welfare \cite{liu2018delayed} and the context where fairness is applied \cite{corbett2023measure}. Finally, there is work such as \citet{patro2022fair} that critiques fairness in a specific domain (ranking) similar to how we critique fairness in clustering. 






\paragraph{Contributions and Organization of the Paper.} We start in Section \ref{sec:review} by reviewing clustering. Specifically, we review and formally introduce the problem of clustering. We highlight the two fundamental applications of clustering, namely in operations research for facility location and in machine learning and data analysis for unsupervised learning. We then provide a brief review of the fair clustering literature. In Section \ref{sec:full_utility} we go through utility and welfare issues in fair clustering and show how welfare could possibly be degraded. Section \ref{subsec:ml_shortcomings} goes over the downstream effects of fair clustering in the machine learning pipeline and highlights many caveats. Section \ref{sec:datasets_practical} goes over dataset and practical application issues. Section \ref{sec:misc} goes over a collection of issues shared by all algorithmic fairness paradigms but with context specific to fair clustering. Finally, in Section \ref{sec:towards_more_impactful} we sketch a path and give suggestions on how to make more impactful work in fair clustering. 

\section{Review of Clustering and Fair Clustering}\label{sec:review}
We start by defining clustering concretely focusing on the most prominent centroid-based objectives\footnote{Note that there are many other variants of clustering including  hierarchical clustering \cite{james2013introduction,dasgupta2016cost}, correlation clustering \cite{bansal2004correlation,demaine2006correlation,zimek2009correlation}, and spectral clustering \cite{von2007tutorial,meila2016spectral}. Although, some have been considered in fair clustering, the centroid-based objectives have been more common, focusing on the centroid-based objectives makes our discussion more concrete. Further, most of our observations hold for these objectives as well.}. Consider a set of points $\Points$ with a distance function $d :\Points^2 \xrightarrow[]{} \mathbf{R}_{\ge 0}$ which defines a metric over the points, then a $k$-clustering chooses a set of at most $k$ centers $S$ ($|S|\leq k$) and an assignment function $\phi: \Points\xrightarrow{} S$ (from points to centers) so as to minimize one of the following clustering objectives: 
\begin{align}
     \text{$k$-center:} & \ \ \min\limits_{S,\phi} \ \max\limits_{j \in \Points} d(j,\phi(j))   \label{eq:kcenter} \\ 
     \text{$k$-median:} & \ \ \min\limits_{S,\phi} \ \sum\limits_{j \in \Points} d(j,\phi(j))   \label{eq:kmedian} \\ 
     \text{$k$-means:}  & \ \ \min\limits_{S,\phi} \ \sum\limits_{j \in \Points} d^2(j,\phi(j)) \label{eq:kmeans} 
\end{align}
Note that in the ordinary (unconstrained) clustering setting $\phi$ simply assigns each point to its closest center but when constraints are imposed on the optimization, points maybe assigned to further away centers to satisfy the constraint. Most notions in fair clustering impose a constraint on the clustering objective making the assignment function $\phi$ non-trivial to find. Further, we emphasize in the above that the set of centers $S$ has a cardinality that is upper bounded by $k$ and not necessarily equal to $k$. 
\subsection{Two Perspectives in Clustering: Operations Research vs Machine Learning}\label{subsec:review_cluster}
There are two fundamental perspectives in clustering which have two distinctly different motivations. In fact, these two motivations have developed in two different communities, namely Operations Research (\OR{}) and Machine Learning (\ML{}). We present an overview of these two motivations and how they differ from one another.   

\paragraph{Operations Research (\OR{}):} In Operations Research, clustering is often referred to as the facility location problem where it dates back to at least the sixties and remains an active area of research \cite{kuehn1963heuristic,hakimi1964optimum,balinski1965integer,galvao1989method,davis1969branch,cabot1970network,francis1967some,vergin1967algorithm,farahani2010multiple,ahmadi2017survey,boonmee2017facility,laporte2019introduction,cui2010reliable,arabani2012facility}. In the \OR{} setting, points represent individuals (or clients) and clustering is used to open a collection of facilities (centers) such as warehouses, fire-stations, hospitals, or schools to service the clients. For an individual $j$ and a clustering solution $(S,\phi)$, one can think of the distance between $j$ and its assigned center $d(j,\phi(j))$ as a measure of $j$'s disutility. Interestingly, this implies that the $k$-center problem \eqref{eq:kcenter} minimizes the max-min or Rawlsian objective \cite{rawls1958justice} whereas the $k$-median \eqref{eq:kmedian} minimizes the utilitarian objective \cite{brandt2016handbook,feldman2006welfare}. Note that as one would expect in the \OR{} setting --even when fairness issues are ignored-- many variants of the problem can be introduced to accommodate well-motivated practical considerations such as imposing an upper bound on the total number of individuals serviced by a facility due to capacity issues \cite{khuller2000capacitated}. Further, it is possible that different choices for the centers would lead to different costs and therefore we would modify the function to be minimized by including a term for the the cost of opening the centers \cite{cornuejols1983uncapicitated}. However, we have focused on the objectives in \eqref{eq:kcenter}, \eqref{eq:kmedian}, and \eqref{eq:kmeans} without further additions for ease of exposition and since these are the objectives which have been predominantly considered in fair clustering.




\paragraph{Machine Learning (\ML{}):} Whereas the purpose of clustering in \OR{} is clear and amounts to minimizing the clustering objective, the purpose in \ML{} is more complicated and ill-defined \cite{shalev2014understanding}. Specifically, in \ML{} clustering is used for unsupervised learning to reveal the structure in the dataset and group similar points together and separate faraway ones. In clustering paradigms which minimize a clustering cost function such as the $k$-means, the clustering cost is only a proxy for revealing the structure of the dataset rather than the end objective. Because the desired objective is ill-defined, various different paradigms  were introduced in the \ML{} clustering literature such as hierarchical clustering \cite{james2013introduction,dasgupta2016cost}, centroid-based clustering (such as $k$-\{center, median, means\}), and spectral clustering \cite{von2007tutorial}. In fact, the remarkable work of \citet{kleinberg2002impossibility} lays down simple and desirable properties one would wish in a clustering paradigm and shows that it is impossibly to satisfy all of them simultaneously. Furthermore, the works of \citet{ben2018clustering} and \citet{von2012clustering} give deep critiques and shortcomings in clustering. We quote the following from \citet{ben2018clustering}:  
\begin{quote}
``different algorithms may yield to dramatically different outputs for the same input sets. In contrast
with other common learning tasks, like classification, clustering does not have a well defined ground truth.''
\end{quote}
Our point here is not that clustering does not provide great utility in machine learning and data analysis. However, it does imply that the ambiguity in clustering in \ML{} can cause the application of fair clustering to have unintended downstream effects that possibly nullify the application of the fair clustering algorithm or even degrade the utilities of the individuals. Section \ref{subsec:ml_shortcomings} discusses these potential \ML{} specific pitfalls. 

\subsection{Brief Review of Fair Clustering}\label{subsec:review_of_fc}
Because of the vast growth in the fair clustering literature it is not easy to give a complete view of all of the work. Therefore, we will give concrete definitions to a sample of the fairness notions that will be relevant for the subsequent parts of the paper. In the case of group fairness, the notions we list below have all received significant attention in the literature.  

First, we introduce some further notation. Let $\Colors$ be the set of all groups (colors) which the given set of points in the dataset $\Points$ belong to. Associate with each point $j \in \Points$ a color $\chi(j) \in \Colors$ which denotes its group membership. For simplicity, we assume that each point belongs to only one group. We now give concrete definitions to some fairness notions. 

\paragraph{Proportional Color Mixing (\GF{}):} This is the most prominent notion in group fair clustering \cite{chierichetti2017fair,bera2019fair,bercea2018cost,ahmadian2019clustering}. The notion constrains the solution to have a proportional representation of the different groups (colors) in each cluster. Since different clusters can have different outcomes associated with them, the proportional representation constraint enforces the notion of disparate impact \cite{feldman2015certifying,CivilRights1991}. In its most general form, \GF{} states that for any center $i \in S$ the following constraint should be satisfied: 
\begin{align}
  \forall h \in \Colors: l_h |C_i| \leq |C^h_i| \leq u_h |C_i|  
\end{align}
where $l_h$ and $u_h$ are proportion bounds and $0\leq l_h \leq u_h \leq 1$. Further, $C_i$ is the set of points assigned to center $i$ and $C^h_i$ is the subset of color $h$. A reasonable choice for the bounds $l_h$ and $u_h$ is to be close to the proportion of color $h$ in the dataset. For example, if half of the dataset is red, then we may set $l_{\text{red}}=0.4$ and $u_{\text{red}}=0.6$. 



\paragraph{Socially Fair Clustering (\SF{}):} This notion is motivated by the disparity in the clustering cost function across the groups. I.e., it is possible that a clustering solution (even if optimal) would be small for one group and large for another. To fix this issue, the works of \citet{makarychev2021approximation,abbasi2020fair,ghadiri2021socially} introduce and solve the following clustering objective: 
\begin{align}
    \max\limits_{h \in \Colors} \frac{1}{|\Points^h|} \sum_{j \in \Points^h} d^p(j,\phi(j)) \label{eq:sf_ob}
\end{align}
where $p=1$ and $2$ for the $k$-median and $k$-means, respectively. 
Note that this fairness notion is stated as a minimization problem without constraints. A solution to such a \SF{} formulation has an objective value that is an multiplicative approximation to optimal solution of the same problem, $\frac{1}{|\Points^h|} \sum_{j \in \Points^h} d^p(j,\phi(j)) \leq \beta \cdot \frac{1}{|\Points^h|} \sum_{j \in \Points^h} d^p(j,\phi^*(j))$. Thus the optimization problem can be equivalently viewed as a constrained problem,  where solutions with small $\beta$ are sought after.

Although there has been work on individual fairness notions in clustering, most of the research in fair clustering had been focused on group fairness notions. We will include a specific notion of equitable fairness that was introduced in \citet{chakrabarti2022new}.  
\paragraph{Equitable Distance Fairness (\EQ{}):} As the name suggests the motivation behind this notion is to guarantee an upper bound on the utility variation between different points. More concretely, each point $j \in \Points$ has a set $S_j \subset \Points$ associated with it and a solution is considered $\alpha$-equitably fair\footnote{Actually, this notion is formally called per-point equitable in 
\citet{chakrabarti2022new} as opposed to average equitable where the average of the distances in the similarity set instead of the minimum is taken in equation \eqref{eq:equitable_def}. We focus on per-point equitable fairness for the sake of clarity and ease of representation.} if the following holds: 
\begin{align}
    \forall j \in \Points: d(j,\phi(j)) \leq \alpha \min\limits_{j' \in S_j} d(j',\phi(j')) \label{eq:equitable_def}
\end{align}

Finally, we point out that for a given instance and a given fairness constraint $c$ (e.g. $c$ could be \GF{} or \EQ{}), the price of fairness (PoF) is defined as $\POF = \frac{\text{Cost of Optimal Solution Satisfying Constraint $c$}}{\text{Cost of Optimal Agnostic Solution}}$. Accordingly, the PoF measures the degradation in the clustering cost due to imposing the fairness constraint.   
\section{How Does Fair Clustering Affect Utility and Welfare?}\label{sec:full_utility} 
A large collection of papers have shown that welfare considerations are of critical importance in fairness settings, i.e. how an algorithm that is purported to be fair would affect the utilities of the individuals \cite{heidari2019long,hu2020fair,mladenov2020optimizing,chen2021fairness,heidari2018fairness,chohlas2021learning}. In fact, \citet{liu2018delayed} and \citet{chohlas2021learning} show that the application of a ``fair'' algorithm could potentially cause harm when the full interaction between the algorithm and the individuals is not taken into account. Following this observation in a clustering setting, we show how the application of various fair clustering notions could potentially cause harm by assuming a very simple and reasonable utility model. In fact, we do not even assume a specific algebraic relation for the utility, only the form of the dependence.

Following the standard model of fair clustering, we treat each point in clustering as an individual. From Subsection \ref{subsec:review_cluster} it is clear that the utility of a point is improved if the distance from its assigned center is made smaller, this holds in both the \OR{} and \ML{} perspectives. From the \OR{} perspective, being closer to the center means that the travel distance is shorter while from the \ML{} perspective being closer to the center means the center is more representative of the point since distance in the \ML{} settings is a measure of dissimilarity. Furthermore, different centers (clusters) can have different outcomes (of varying qualities) associated with them\footnote{\citet{esmaeili2022fair} refers to these different outcomes as ``labels.''}. For example, in \OR{} centers (which may represent schools or facilities) could provide services of different levels of quality \cite{xu2008approximation,shmoys2004facility}. This is the case in the \ML{} setting as well;  consider the use of clustering for a market segmentation application where different clusters could advertise for jobs of varying levels of payment \cite{esmaeili2022fair,aggarwal2004method,chen2012data,han2011data,tan2018introduction}. Furthermore, the outcome of the center (cluster) may not be fixed but may depend on the set of points assigned to it. For example, if the centers represent schools then an assignment of points that is more diverse across demographic groups would be more preferable \cite{bercea2018cost}.   


Therefore, in general the utility a point gains in a clustering $(S,\phi)$ can be reasonably approximated by: 
\begin{align}
    u_j(S,\phi) = f_j\Big(d\big(j,\phi(j)\big), L\big(\phi,j\big)  \Big) \label{eq:util_of_clustering}
\end{align}
Where $f_j$ is a two-input function. $L(\phi,j)$ is the outcome associated with the center (cluster). Importantly, for a fixed value of $L(\phi,j)$, $f_j$ is a \emph{decreasing} function in $d\big(j,\phi(j)\big)$. Notice the subscript $j$ in $f_j$ which implies that in general different points (individuals) can have different preferences which is an important consideration as pointed out by \citet{finocchiaro2021bridging}. The welfare of all individuals could then be aggregated using the utilitarian objective \cite{brandt2016handbook,feldman2006welfare}, which would be the sum of the utilities of the individuals leading to:
\begin{align}
    U(S,\phi) = \sum_{j \in \Points} u_j(S,\phi) \label{eq:util_of_clustering_group_full}
\end{align}

The welfare of a specific group (color) is dependant on the utilities of its points. Accordingly, a specific group $h$ would have the following average welfare: 
\begin{align}
    U_h(S,\phi) = \frac{1}{|\Points^h|} \sum_{j \in \Points^h} u_j(S,\phi) \label{eq:util_of_clustering_group}
\end{align}
We are not aware of a fair clustering formulation that quantifies welfare with the exception of \citet{abbasi2023measuring}. However, the work of \citet{abbasi2023measuring} is focused on a specific application and does not consider the outcome heterogeneity that could exist between different centers (i.e., some centers being better than others). Further, our objective here is general as we intend to show how the introduced fairness notions would effect welfare in light of the above model. Accordingly, we will discuss a collection of ignored issues that surface once utility considerations are more carefully taken into account. 

In the following subsections, we will show through illustrative examples how welfare could be degraded because of the application of a fair clustering algorithm. Specifically, we show that the entire welfare $U$ could be degraded and that also sometimes it  could be degraded for a specific group (possibly the protected group). In fact, multiple optimal fair solutions could exist that result in different distribution of welfare across the groups. We demonstrate these observations for specific fairness constraints through simple examples with a small number of points and two colors but similar observations can be extended to more complicated examples with more points and colors and for other fairness constraints. Note in the examples that when we impose the \GF{} constraint, we only have two colors (blue and red) and for simplicity set the upper and lower bounds in \GF{} equal to $\frac{1}{2}$.  


\subsection{Fair Clustering Could Degrade Welfare}\label{subsec:welfare_deg} 
\citet{chierichetti2017fair} (arguably the founding paper of fair clustering) indicates that points maybe assigned to further away centers to satisfy the fairness constraints. Therefore, in light of the utility model of \eqref{eq:util_of_clustering} which shows that welfare is dependent on both the distance from the center as well as the outcome associated with it, it is worthwhile to wonder if the welfare is actually degraded by the application of fair clustering algorithms since fairness constraints are myopic and do not have a full view of  the utility.


\paragraph{\GF{} Constraint Example.} See Figure \ref{fig:welfare_ex_gf_1} which contrasts the agnostic (unfair) clustering output with the \GF{} clustering output over a set of points that belong to two groups. Agnostic clustering leads to clusters $C_1$ and $C_3$ to be composed entirely of blue points. Therefore, the outcome associated with clusters $C_1$ and $C_3$ which could be of higher quality will not bring any benefit to the red group since no red point is included in them. Even if all of the clusters had the same outcome associated with them, the lack of diversity (group-wise) in clusters $C_1$ and $C_3$ is not satisfactory especially in an \OR{} application like school assignment. The \GF{} fairness constraint was motivated by such examples and imposing it would lead to an output where all groups are well-represented in each cluster. However, as can be seen from the output most blue and red points have to travel a much larger distance in the new \GF{} clustering. As the distance becomes sufficiently large, one can conclude that the welfare of the blue group $U_{\text{blue}}$ as well as that of the red group $U_{\text{red}}$ has indeed been degraded because of the application of the \GF{} constraint and accordingly the entire welfare $U$ has been degraded.

\begin{figure}[h!]
   \centering
   \includegraphics[scale=0.6]{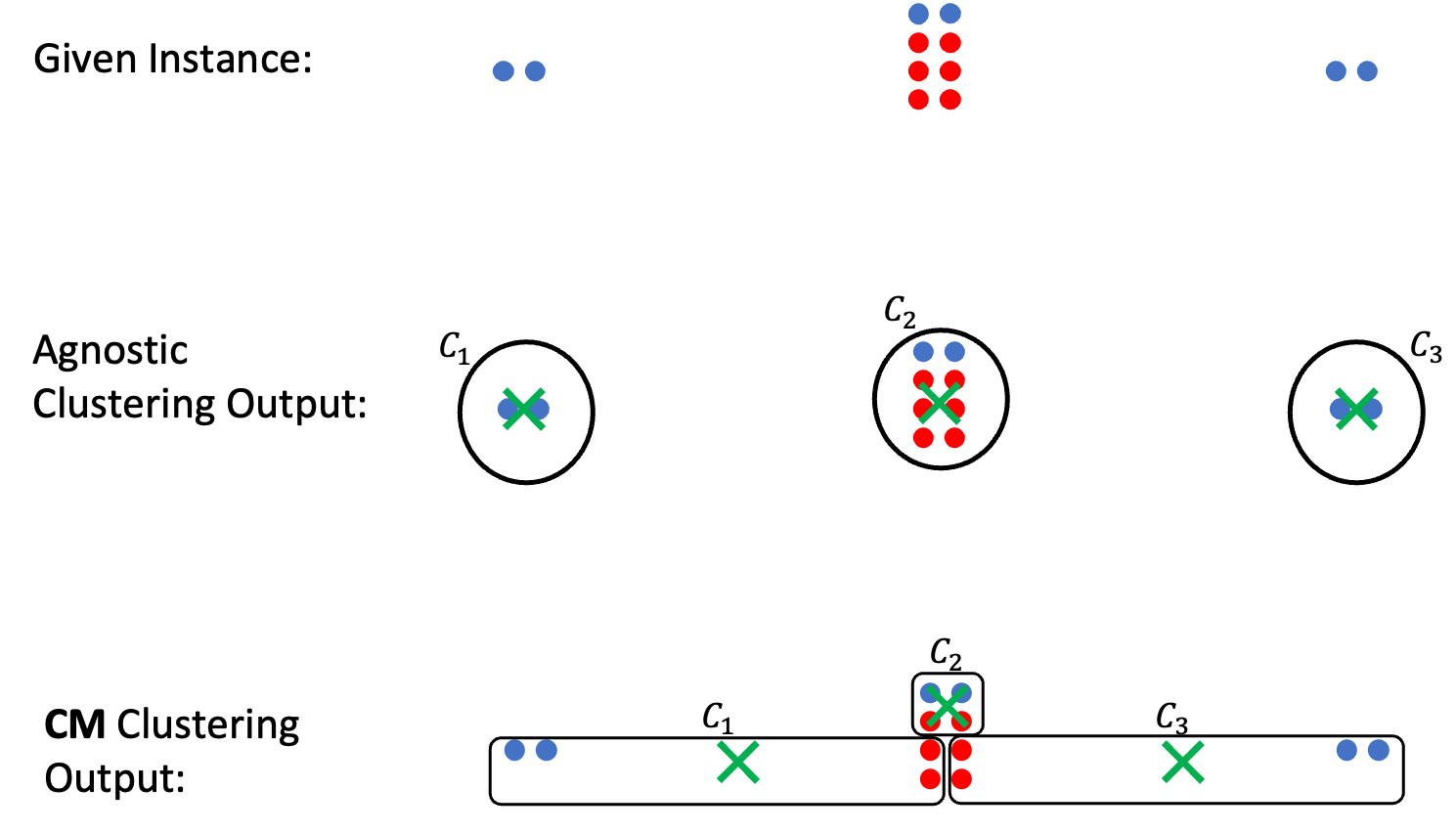}
   \caption{The figure shows an instance with the agnostic vs the \GF{} clustering output. Note that centers are labeled by a green marker \textbf{X}.}
   \label{fig:welfare_ex_gf_1}
\end{figure}

Note that such a behaviour could happen in real life. In particular, in the case of clustering for school assignment while we would end up with a balanced group representation in each school, the distance travelled by many students could be quite large. Given that racial memberships correlate significantly with geographic location \cite{fiscella2006use,long2021use}, this issue is practically well-motivated.




\paragraph{\EQ{} Constraint Example.} Here we consider the \EQ{} constraint which is an individual fairness constraint and assume a clustering that only chooses centers from the given set of points. See Figure \ref{fig:welfare_ex_equitable} where agnostic clustering recovers the true structure in the dataset, clustering nearby points and separating ones that are far away from each other. However, equitable clustering results in a very different clustering. First, we note that points 1 and 2 are in each other's similarity sets and likewise points 4 and 5 are in each others similarity sets whereas point 3 is only similar to itself. Note further each of the pair $\{1,2\}$ and $\{4,5\}$ are at a small distance of $\epsilon$ from each other. One can verify from the definition of equitable clustering as shown in Inequality \eqref{eq:equitable_def}) that the \EQ{} solution show in \ref{fig:welfare_ex_equitable} is optimal for the $k$-center objective. Note however, that it does not allow points $\{1,2\}$ or $\{4,5\}$ to form a cluster and instead all are assigned to point $3$ in the middle forming only one cluster. As a result the point-to-center distances become much larger and similar to the \GF{} example the welfare $U$ could indeed be degraded when compared to the agnostic clustering. 
\begin{figure}[h!]
   \centering
   \includegraphics[scale=0.6]{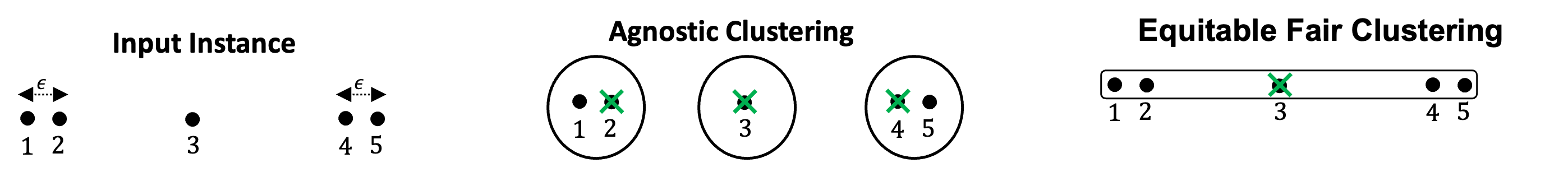}
   \caption{The figure shows an instance with the agnostic vs the \EQ{} clustering output. Note that centers are labeled by a green marker \textbf{\textcolor{green}{X}}..}
   \label{fig:welfare_ex_equitable}
\end{figure}

\paragraph{\SF{} Constraint Example.} In this example we will consider the \SF{} constraint which ignores the outcome associated with the cluster and the within cluster diversity level and show how agnostic clustering could lead to a higher welfare. See the example of Figure \ref{fig:ex_socially_fair} where the application of agnostic clustering leads both clusters to have population-level proportional representation of each group. This implies that both clusters have a good level of diversity and that the different groups will attain the same outcome associated with each cluster in equal proportions. That is not the case however when applying the socially fair clustering notion of \cite{abbasi2020fair,ghadiri2021socially,makarychev2021approximation}. If the top centers (which only includes blue points in the socially fair case) receive an outcome that is highly desirable, then it is possible that the red points at the bottom would gain a higher utility from the application of an agnostic instead of a socially fair clustering since none of them are included in a top center in the \SF{} solution.

\begin{figure}[h!]
   \centering
   \includegraphics[scale=0.6]{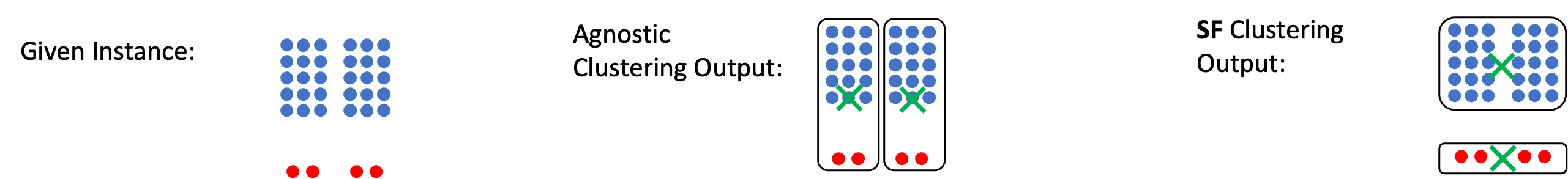}
   \caption{The figure shows an instance with the agnostic vs the \SF{} clustering output. Note that centers are labelled by a green marker \textbf{\textcolor{green}{X}}.}
   \label{fig:ex_socially_fair}
\end{figure}

\paragraph{Additional Remarks:} Note that all of the examples mentioned are not pathological clustering examples (from the agnostic prospective). In fact, in the case of the \GF{} constraint (Figure \ref{fig:welfare_ex_gf_1}) and \EQ{} constraint (Figure \ref{fig:welfare_ex_equitable}) the clusters consist of points close to each other with high inter cluster separation.    
Moreover, the literature has mostly ignored such issues and while it is true that when a notion of fair clustering is used, the price of fairness ($\POF$) is usually measured. The $\POF$ is measured according to the degradation in the clustering cost, not the degradation in overall utility for the individuals or groups. 
Further, a more rigorous and justified approach to fair clustering would consider both distance and outcome simultaneously. Finally, it is possible that one may encounter a situation where both considerations are important to take into account to improve welfare. For example, points would be routed to centers further away only if the centers are not too far or if the outcome associated with the center is preferred. In fact, a much more preferred method would give a clear and well-justified description of the function $f_j\Big(d\big(j,\phi(j)\big), L\big(\phi,j\big)  \Big)$ in equation (\ref{eq:util_of_clustering}).

\subsection{Inequitable Welfare Degradation}\label{subsec:deg_focused}
Here we show a perhaps surprising issue which is that a fair clustering solution could be unfair when it comes to the degradation in the
the welfare. At the extreme, a fair clustering solution may assign points of a specific group a large distance to their center while  points from other groups can have their distance essentially unchanged. This issue has not be highlighted in the literature and in fact the degradation in clustering cost ($\POF$) is mostly never measured for each group separately. Figure \ref{fig:ex_2_agnostic} shows an interesting example where imposing the \GF{} constraint could lead to two different optimal solutions. However, in the first solution red points are assigned to further away centers while in the second solution blue points are instead assigned to further away centers. Although, this is an extreme example, one can show other examples where the degradation in the clustering cost for one group is higher than the other and for fairness notions other than \GF{}. The fact that the welfare degradation may not be equitable across the groups puts into question the fairness of the solution. 
\begin{figure}[h!]
   \centering
   \includegraphics[scale=0.4]{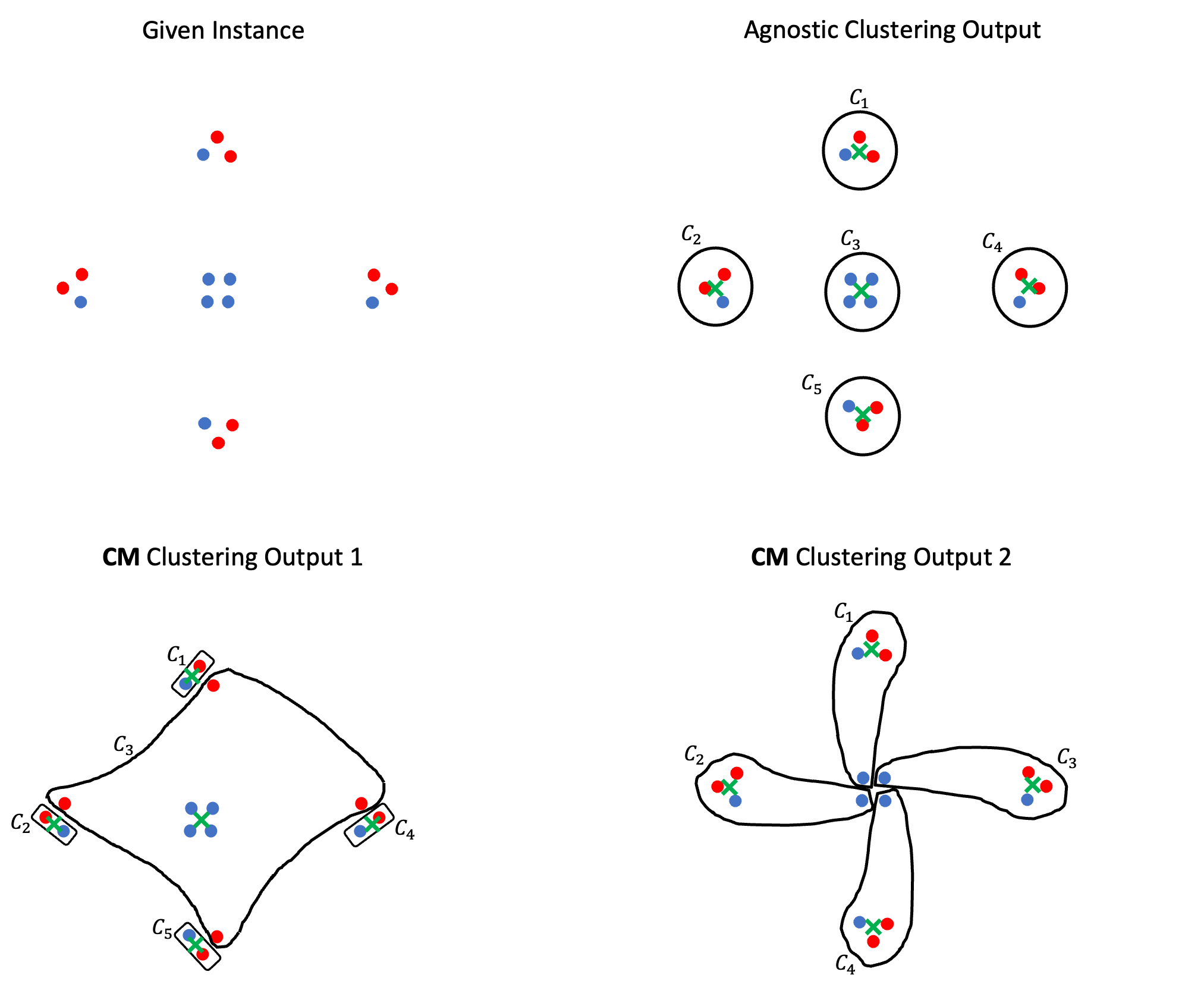}
   \caption{In this example points which are nearby points (the four triads and the four blue point middle points) are separated by a small distance of $\epsilon$ whereas every other distance between any two points is at least $R \gg \epsilon$. Although the two \GF{} clustering solutions in the bottom row have approximately equal clustering cost they result in different distance assignments for the red and blue groups. The first is favorable to the blue group wherease the second is favorable to the red group.}
   \vspace{-0.2cm}
   \label{fig:ex_2_agnostic}
\end{figure}

\subsection{Maximizing Welfare: Going Beyond Simple Constraints}
Building on the previous discussion we show a natural example (see Figure \ref{fig:welfare_example}) where the individuals value within cluster diversity as well as short travel distance. However, the trade-off between diversity and travel distance varies across different regions. Specifically, in one region diversity can be achieved at the expense of a short travel distance whereas in another it can only happen at the expense of a large travel distance. Therefore, a \GF{} and an \SF{}\footnote{Note that the optimal \SF{} clustering in this example would also be equal to the agnostic clustering} clustering would result in sub-optimal utility for each group and overall. Whereas this would not be the case using a welfare-centric notion (which maximizes \eqref{eq:util_of_clustering_group}) since it would essentially give a \GF{} clustering where the diversity expense is small and an \SF{} clustering where the diversity expense is large. This highlights a draw back in using simple fairness notions. For completeness, we establish this formally in the form of the theorem shown below. Note that the theorem holds under a reasonable choice for the utility $u_j(S,\phi)=f_j(d(j,\phi(j)),L(\phi,j))$.   
\begin{figure}[h!]
    \centering
    \includegraphics[width=1\linewidth]{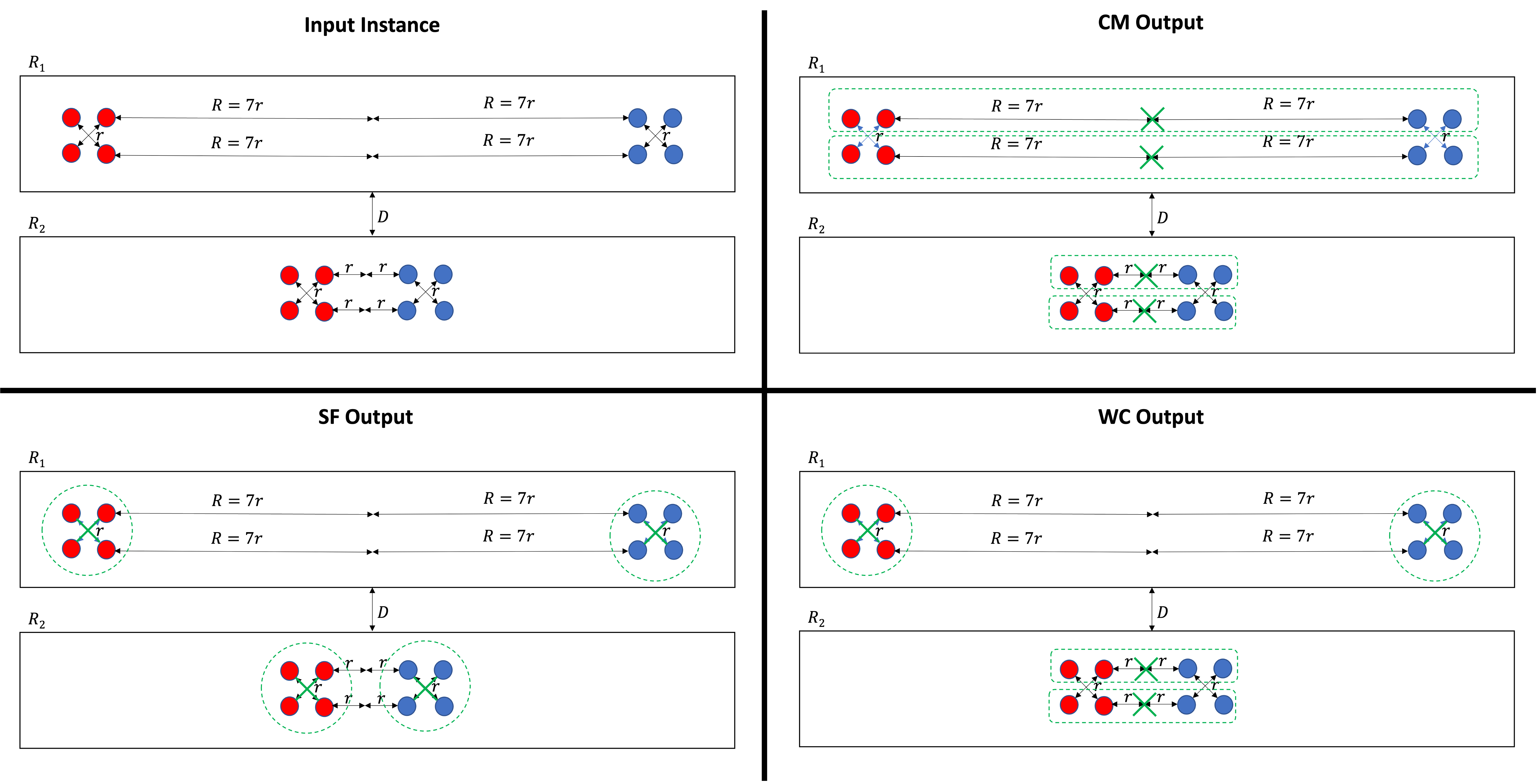}
    \caption{The figure shows the input instance consisting of two regions $R_1$ and $R_2$ separated by a very large distance $D$ ($D$ is shown smaller in the figure to save space). The resulting clusterings for \GF{}, \SF{}, and the welfare-centric \WC{} clusterings are all shown with the clusters enclosed by dashed lines and centers with green \textcolor{green}{\textbf{X}}. Note how \WC{} gives the most natural solution which is a mixture of both \GF{} and \SF{}, achieving diversity only when it comes at a reasonable expense.}
    \label{fig:welfare_example}
\end{figure}

\begin{restatable}{theorem}{welfareExample} \label{thm:welfare_example}
In the instance shown in Figure \ref{fig:welfare_example}, for the $k$-median problem with $k=4$ a \GF{} or an \SF{} clustering would have an average utility of at most $2r$ for each group whereas a welfare-centric clustering would result in an average utility of at least $3 r$ where $r$ is a positive number. 
\end{restatable}
\begin{proof}
\textbf{Setting the utility Value:} First, we define the utility of a point $j$. We set the utility to the following 
\begin{align}
    u_j(S,\phi) = \Big( 3r-d(j,\phi(j)) \Big) + \Big( 3r \cdot \min\{ \frac{|S^{\text{red}}_{\phi(j)}|}{|S^{\text{blue}}_{\phi(j)}|} , \frac{|S^{\text{blue}}_{\phi(j)}|}{|S^{\text{red}}_{\phi(j)}|} \} \Big)
\end{align}
Now, we highlight some details about the utility. The first term $\Big( 3r-d(j,\phi(j))\Big)$ is for the distance and is non-negative as long as $d(j,\phi(j)) \leq 3r$. The second term is concerned with the diversity in the cluster, note that $S_{\phi(j)}$ is the cluster point $j$ is assigned to and $S^{\text{red}}_{\phi(j)}$ and $S^{\text{blue}}_{\phi(j)}$ are the subset of red and blue points within the cluster, respectively. Further, $\min\{ \frac{|S^{\text{red}}_{\phi(j)}|}{|S^{\text{blue}}_{\phi(j)}|} , \frac{|S^{\text{blue}}_{\phi(j)}|}{|S^{\text{red}}_{\phi(j)}|} \}$ is a measure of diversity within the cluster obtaining a maximum value of $1$ when the red and blue points are equally represented and a minimum value of $0$ when the cluster consists of only one group. The $3r$ is a scaling parameter for the diversity, hence the final value of the second term is $\Big( 3r \cdot \min\{ \frac{|S^{\text{red}}_{\phi(j)}|}{|S^{\text{blue}}_{\phi(j)}|} , \frac{|S^{\text{blue}}_{\phi(j)}|}{|S^{\text{red}}_{\phi(j)}|} \} \Big)$.

\textbf{Upper bound on the utility of \GF{} clustering:}
The upper bound on the utility for any point $j$ in $R_1$ for a \GF{} clustering is 
\begin{align}
     u_j(S_{\GF},\phi_{\GF})      & \leq \Big(3r- R \Big) + \Big( 3r \cdot 1 \Big) \\ 
                                  & = \Big(3r-7r \Big) + \Big( 3r \Big) \\ 
                                  & = -r 
\end{align}
Now for any point $j$ in $R_2$ the upper bound is 
\begin{align}
    u_j(S_{\GF},\phi_{\GF})      & \leq \Big(3r-r\Big) + \Big( 3r \cdot 1 \Big) = 5r 
\end{align}
Since both regions have an equal number of points from each group the average is at most 
\begin{align}
    \frac{-r + 5r}{2} = \frac{4r}{2} = 2 r 
\end{align}
Therefore, we have 
\begin{align}
    U_{\text{red}}(S_{\GF},\phi_{\GF}) , U_{\text{blue}}(S_{\GF},\phi_{\GF}) \leq 2r
\end{align}
\textbf{Upper bound on the utility of \SF{} clustering:}
The upper bound on the utility of an \SF{} clustering for any point $j$ in $R_1$ or $R_2$ is 
\begin{align}
     u_j(S_{\SF},\phi_{\SF}) \leq \Big(3r-r\Big) + \Big( 3r \cdot 0 \Big) = 2r 
\end{align}
Therefore, we have 
\begin{align}
    U_{\text{red}}(S_{\SF},\phi_{\SF}) , U_{\text{blue}}(S_{\SF},\phi_{\SF}) \leq 2r
\end{align}

\textbf{Lower bound on the utility of the welfare-centric clustering:}
The welfare-centric clustering \WC{} on the other hand would maximize the following objective:
\begin{align}
    \max\limits_{S,\phi} \min\limits_{h \in \Colors} U_h(S,\phi) 
\end{align}
where $U_h(S,\phi) =\frac{1}{|\Points^h|} \sum_{j \in \Points^h} u_j(S,\phi)$ as defined in \eqref{eq:util_of_clustering_group}. I.e., \WC{} maximizes the minimum average utility across groups. \WC{} has the same clustering as \SF{} in the first region $R_1$ and the same clustering as \GF{} in the second region $R_2$. In the first region $R_1$ the utility of any point $j$ will be 
\begin{align}
     u_j(S_{\WC},\phi_{\WC}) = \Big(3r-r\Big) + \Big( 3r \cdot 0 \Big) = 2r 
\end{align}
In the second region the utility of a point will be at least 
\begin{align}
     u_j(S_{\WC},\phi_{\WC}) \ge \Big(3r-2r\Big) + \Big( 3r \cdot 1 \Big) = 4r 
\end{align}
This makes the average utility of any group at least  
\begin{align}
    \frac{2r+4r}{2} = \frac{6r}{2} = 3r 
\end{align}
Therefore, we have 
\begin{align}
    U_{\text{red}}(S_{\WC},\phi_{\WC}) , U_{\text{blue}}(S_{\WC},\phi_{\WC}) \ge 3r
\end{align}
\end{proof}

\section{Caveats of Fair Clustering: Unintended Downstream Effects in \ML{} Settings}\label{subsec:ml_shortcomings}


\begin{figure}[h!]
  \begin{center}
    \includegraphics[width=0.8\linewidth]{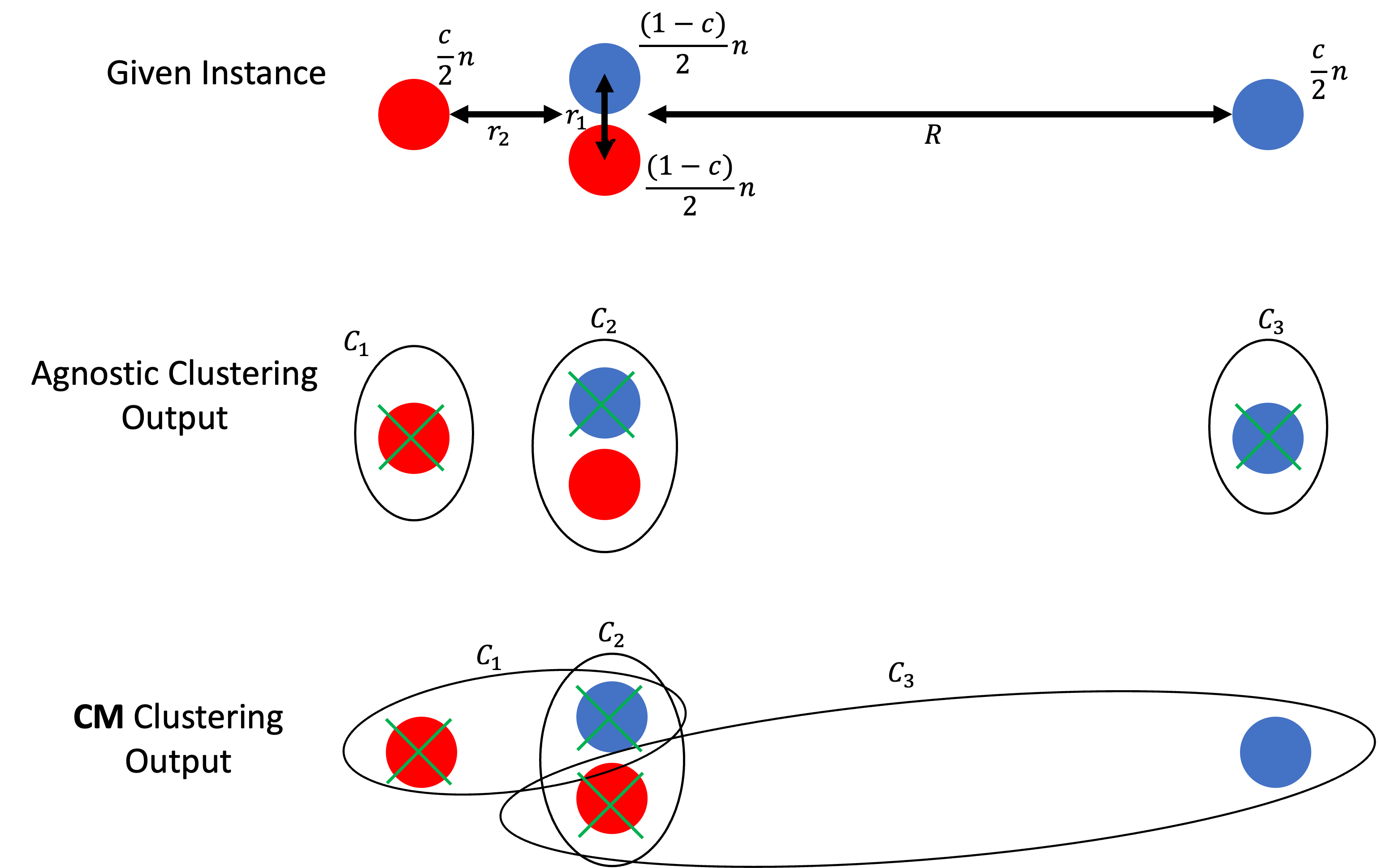}
  \end{center}
  \caption{In this example we have a set of points that coincide in the same location: both sets in the middle consist of $\frac{1-\anomeps}{2}n$ points each and the other two sets on the left and right consist of $\frac{\anomeps}{2}n$ points each where $c<\frac{1}{2}$. The middle blue and red sets are separated by a very small distance $r_1$ while their distance from the left set is approximately $r_2$ with $r_2 \ge r_1$. On the other hand, the blue set on the right is separated from the middle sets by at least $R$ and we have $R \gg r_1,r_2$. Note that in the \GF{} clustering the clusters are overlapping since they include points from coinciding sets.}
  \vspace{-0.1cm}
  \label{fig:anomaly_ex}
\end{figure}

We will focus in this section on the application of fair clustering in \ML{}. Our concern here is not primarily with the welfare of the groups but the validity of some methods used in \ML{} now that fair clustering is used instead of ordinary clustering. More specifically, a fair clustering may produce clustering outputs that differ significantly from a traditional clustering and therefore my lead to unintended downstream effects. As stated in Subsection \ref{subsec:review_cluster}, in machine learning and data analysis a clustering of a dataset is a partitioning of it into groups (clusters) where points in the same cluster are supposed to be similar to one another and points from different clusters are supposed to be dissimilar from one another. However, many of the fairness notions in clustering may group points that are faraway from each other to satisfy the fairness constraint as discussed and shown in examples in Section \ref{sec:full_utility}. This issue is not unique to fairness in clustering but can in general be seen in constrained clustering. For example, imposing an upper bound on the total number of points in a cluster may lead to similar behaviour since points in dense regions may need to be routed to centers further away in order not to violate the upper bound on the total number of points in a cluster \cite{barilan1993allocate,khuller2000capacitated,cohen2022fixed,aggarwal2010achieving,rosner2018privacy}. While imposing an upper bound is well-motivated in \OR{} settings as it would correspond to the service capacity of the facility, the same is not necessarily true in machine learning where one wants to reveal the structure of the dataset. In fact, one cannot think of many modifications to a clustering objective that are well-aligned with the \ML{} clustering desideratum --of grouping similar points and separating dissimilar ones-- that would assign points to further away centers. Therefore, given the fact that a fair clustering may group distant points in the same cluster it is worthwhile to wonder if classical post-processing methods that are applied to clustering --which are well-justified in an ordinary \ML{} setting-- remain well-justified when fairness has been imposed on the clustering. This is discussed in the following subsections. 




\subsection{False Positives and False Negatives in Outlier Detection}\label{ex:outlier}

Given a clustering, the data analyst may choose to use it to detect or remove outlier data points. A method that is well-known in clustering-based outlier detection is to flag points that are faraway from their centroid as anomalies \cite{chandola2009anomaly,smith2002clustering}. In ordinary (unconstrained) clustering, the point is assigned to its closest center. However, as mentioned earlier a fair clustering may assign points to further away centers to satisfy the fairness constraints. Therefore, if the data analyst chooses to apply such an outlier detection method over the output of a fair clustering using the distance of a point to its assigned center then she may flag points as anomalies when in fact they are not. One may think that this could be fixed by using the distance between the point and its closest center instead of its assigned center in the fair clustering. However, the center chosen by a fair clustering algorithm could be different from the center chosen by an ordinary algorithm. Furthermore, as mentioned in Subsection \ref{subsec:deg_focused} the points assigned to faraway centers in fair clustering might have a high representation from a specific group, leading a specific group to be disproportionately flagged as outliers. Such an outcome could be considered as causing harm. In Figure \ref{fig:anomaly_ex} we show an example that exhibits the above behavior --for simplicity we assume that centers have to be selected from the given set of points which is the case in many practical applications-- where points belonging to a specific group are abnormally faraway from their center and could be flagged as outliers. More Specifically in the figure, agnostic clustering would lead to all points being away from their center by a very small distance of at most $r_1$ while the \GF{} clustering would lead the set of blue points on the right to be at a large distance of $R$ from their assigned center which is much larger than the rest of points and therefore they are very likely to be flagged as outliers. Note that this happens in an optimal \GF{} clustering.


\begin{figure}
  \begin{center}
    \includegraphics[width=0.8\linewidth]{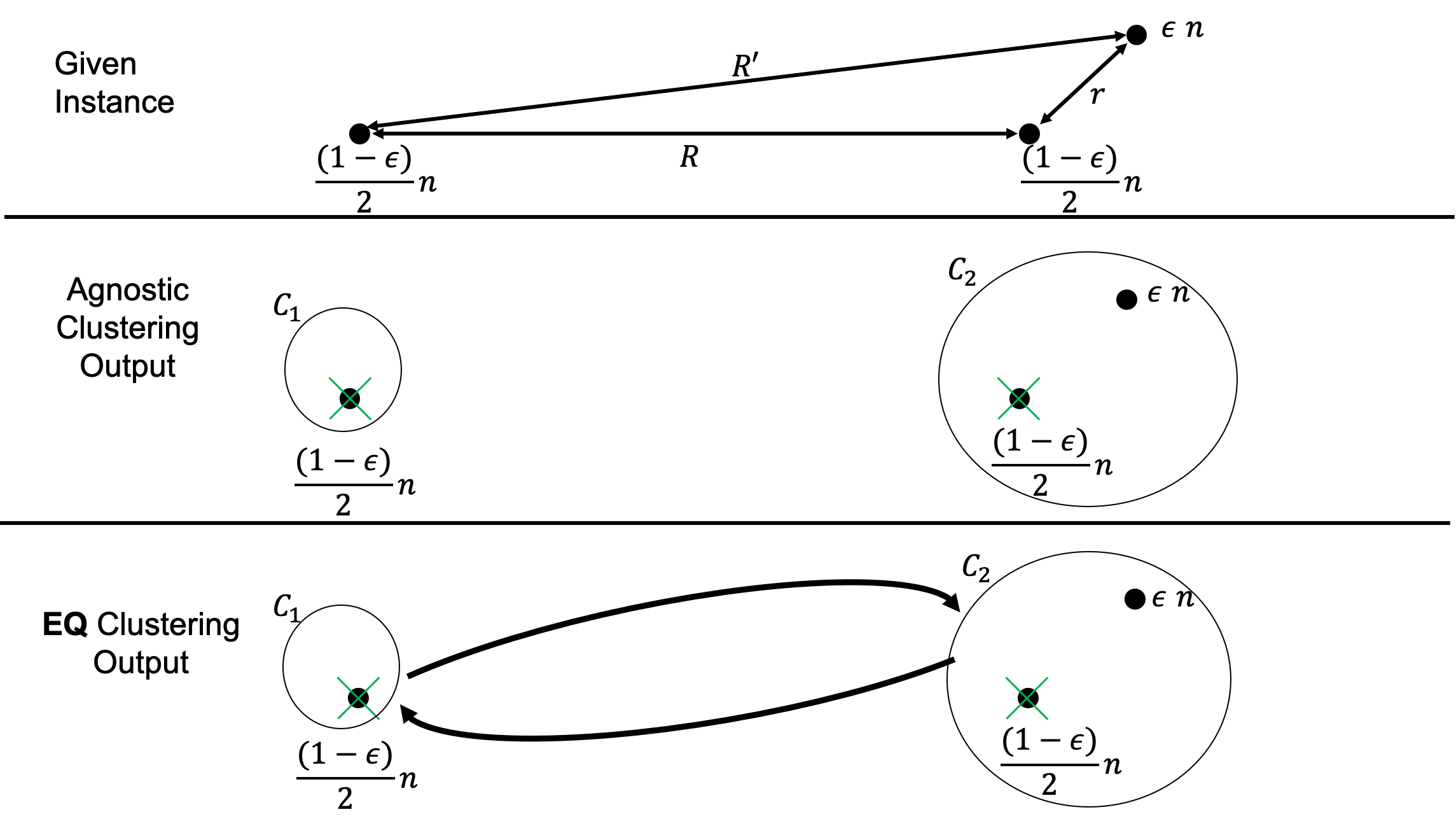}
  \end{center}
  \caption{We have three sets of coinciding points, the set on the top right consist of $\epsilon n$ many points while the remaining two sets consist of $\frac{1-\epsilon}{2}n$ each. The similarity set of any point includes the entire dataset. The distances are shown, note that we set $\frac{R'}{R} \leq \alpha$.}
  \vspace{-0.1cm}
  \label{fig:anomaly_equitable_ex}
\end{figure}

Furthermore, the equitable fair clustering notion \EQ{} of \citet{chakrabarti2022new} (see Subsection \ref{subsec:review_of_fc}) forces the maximum distance ratio between points in the same similarity set to be at most $\alpha$. While this might be desirable in some applications, the clustering output may not be useful for the outlier-detection application mentioned above since the difference in distances has been significantly reduced. In Figure \ref{fig:anomaly_equitable_ex} we show an example where applying agnostic clustering (we assume a $k$-median or $k$-means objective and that centers have to be selected from the given set of points) would result in the top right set of $\epsilon n$ many points that are clearly faraway from the rest of the cluster center to be possibly flagged as outliers. In a practical application that might in fact be the right choice since the rest of the cluster on the right $C_2$ are at a distance zero from the center. On the other hand, the \EQ{} clustering output would have the same set of centers but would assign all of the points on the right to the left and vice versa. Note that for each point the similarity set consists of all points in the dataset. Now the points in the right cluster $C_2$ have a much smaller distance to center ratio of at most $\frac{R'}{R}$ and based on distance to center the top right set of points may incorrectly be consider as ordinary (non-outlier) points especially if $\frac{R'}{R}$ is small.

The above examples show interesting effects that could result from a fair clustering for outlier detection. The first leads to \emph{false positive} outliers (\GF{} constraint, Figure \ref{fig:anomaly_ex}) while the second leads to \emph{false negatives} (\EQ{} constraint, Figure \ref{fig:anomaly_equitable_ex}). The above issue could possibly be fixed, by for example, using an agnostic clustering output when doing anomaly detection. But this highlights the fact that a fair clustering is not a clustering in the traditional \ML{} sense. The downstream effects of a fair clustering should be taken into account more carefully. It is also worthwhile to mention the line of work on clustering with outliers where a subset of the points (to be chosen by the clustering algorithm) are ignored when calculating the clustering cost \cite{feng2019improved,chawla2013k,harris2019lottery,krishnaswamy2018constant}. While \citet{almanza2022k} extends this line of work to take group fairness considerations into account by having a proportional guarantee on the number of points chosen as outliers from each group, it still does not resolve the above issue since the resulting clustering does not additionally combine a desired notion of fairness such as \GF{} or \EQ{}.



\subsection{Leveraging Cluster Homogeneity for Using Simpler \ML{} Methods}\label{ex:clustering_classofication} 
Since a clustering partitions the dataset into homogeneous groups, this homogeneity within the cluster can be used to apply simpler \ML{} methods. For example, in supervised learning having clustered the dataset the data analyst may choose to use a specific classifier for each cluster. Since the cluster consists of similar points which are close in the metric space, then this may allow the usage of simpler and more tractable models such as a linear classifier/regressor. In an interactive learning setting such as in multi-armed bandits, one may use contextual bandits where the cluster decides the context as noted in \citet{lattimore2020bandit}. However, if the clustering is the output of a fair instead of an ordinary (unfair) clustering then since points maybe assigned to centers that are further away, this puts into question the validity of such approaches. More specifically, the clustering could merge points which are far away from one another and are separated in the feature space. Since this is a possible outcome, the points may not have a similar correlation with the desired output label. In Figure \ref{fig:linearclassification}, below we show an example where one can do an agnostic clustering of the dataset and  use a linear classifier for each cluster to obtain a zero classification error. However, using a fair clustering output (such as a \GF{} fair clustering) a linear hypothesis class would not lead to a small error classifier since blue points would necessarily have to be merged with more red points and the separating lines are different and clearly from the figure they have a different linear classifier.

Therefore, this simple and common approach may not provide the expected value if applied over a fair clustering. Similar to the previous subsection the above highlights how a fair clustering may behave differently in an \ML{} pipeline and therefore require caution or special treatment by the data analyst. 



\begin{figure}[h]
    \centering
    \includegraphics[width=\linewidth]{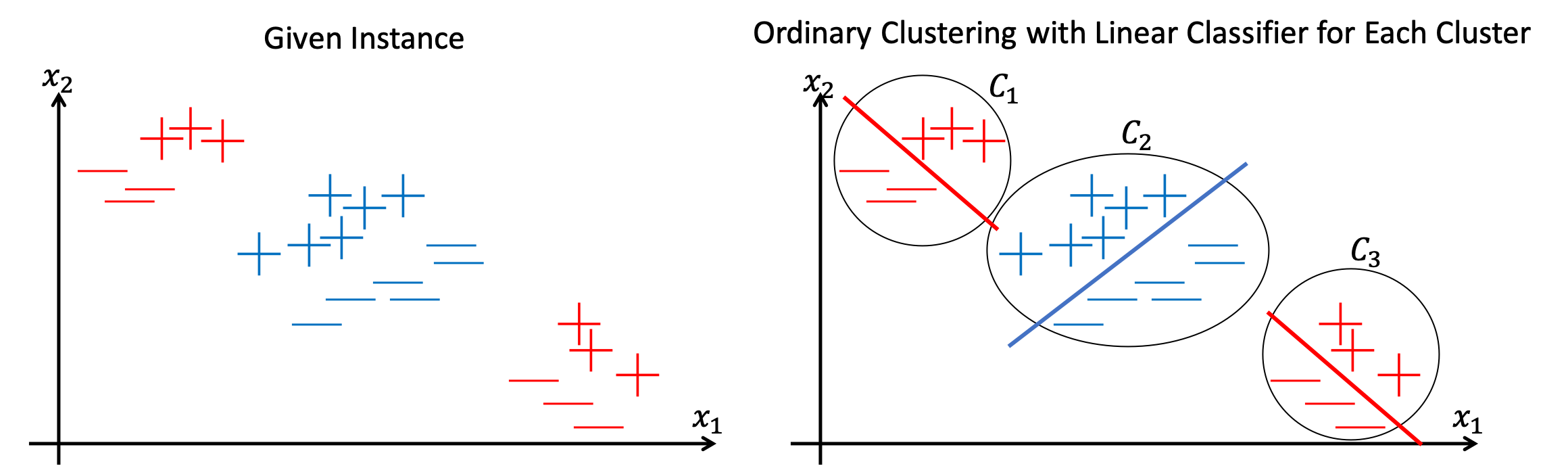}
    \caption{The example shows a collection of points in the feature space belonging to two classes $\{+,-\}$. By applying ordinary clustering (with $k=3$) followed by a separate linear classifier for each cluster we can obtain a zero classification error. However, since a \GF{} fair clustering would have to merge red and blue points in a $50\%-50\%$ proportion in each cluster a separate linear classifier for each cluster we would not obtain a low error since the majority of the red and blue points have a different separating line between the $+$ and $-$ classes.}
    \vspace{-0.1cm}
    \label{fig:linearclassification}
\end{figure}


\subsection{Ambiguity of the Affects of Clustering on the Final Outcomes}\label{subsec:ambig}
A common usage of clustering is exploratory data analysis \cite{jain1988algorithms,dubes1980clustering,jain1999data}. An analyst may cluster the data set, inspect the prototypical vectors (centroids) as well as the points of each cluster for better understanding, and then make further decisions based on this inspection. The decisions that the analyst may choose are wide and varied. Like in the above, the analyst may apply outlier detection or use a specific classifier for each cluster. The analyst may also find some clusters to be more complicated and warrant further processing such as further data collection within the cluster-associated feature space or she may conclude that this cluster should undergo some denoising process.   

Accordingly, it is common for clustering to be in the beginning of the \ML{} pipeline and to be followed by further (possibly elaborate) steps. This implies that the downstream effects of any fair clustering algorithm in \ML{} are not fully characterized unless the subsequent steps are detailed and clarified. Note that unlike differential privacy \cite{dwork2014algorithmic}, the work of \citet{dwork2018fairness} has shown that in general fairness composition does not hold. I.e., the sequential application of fair algorithms does not necessarily preserve fairness. Therefore, one should not expect that applying a fair algorithm over a fair clustering would necessarily preserve fairness.

\section{Datasets, Experimental Methods, and Impact Considerations}\label{sec:datasets_practical}
In any application (especially \ML{}) the chosen datasets can have critical consequences. One can easily reach wrong conclusions about the behavior of an algorithms or its impacts on individuals by using datasets that are not well-aligned with the application domains of the algorithms or by following unsuitable experimental methods. In this section we elaborate on these issues. 


\paragraph{Common Weaknesses in the Experimental Methods.} (1) \emph{Limited and Weak Datasets:} Most of the literature has used datasets from the UCI repository \cite{frank2010uci}.  Additionally Amazon co-purchase dataset is used in \citet{ahmadian2019clustering} and \citet{ahmadi2020fair_correlation}, and FriendshipNet and DrugNet datasets are used in \citet{kleindessner2019guarantees}. While some of these datasets were in fact intended for clustering tasks such as network discovery, it is not clear that these datasets are all suited for clustering and that clustering algorithms suitable to them are used. For example, as noted by \citet{von2012clustering} clustering in \ML{} should be context-dependent and not thought of as a pure mathematical optimization problem. Yet in fair clustering papers we do not in general find thorough discussions of datasets that goes beyond high-level information such as number of entries and selected features. (2) \emph{Unsuitable and Unjustified Experimental Choices:} Even if we were to assume that the used datasets were in fact suitable for clustering, as noted by \citet{ben2018clustering} a fundamental issue in clustering is that different clustering algorithms can lead to dramatically different outputs. 
In existing literature, we find that fair clustering papers would use different algorithms over the same dataset, e.g. the UCI Bank dataset is used in both \citet{bera2019fair} and \citet{knittel2023generalized} although the first uses the $k$-means algorithm whereas the second uses hierarchical clustering. Second, even if we were to assume that $k$-means, $k$-median, or $k$-center is the right clustering objective for the given dataset, many papers \cite{chierichetti2017fair, bera2019fair, esmaeili2020probabilistic,esmaeili2021fair} use a set of values for the number of centers $k$ to demonstrate the validity of the theoretical guarantees. However, empirically the dataset would have an ``instrinsic'' number of clusters $k$ which would correspond to the true number of clusters. This puts into question, the conclusions that one may draw about the fair and even the ordinary (unfair) clustering. Another issue is that datasets used contain numerical and categorical features and some features are omitted in the experimental procedure. Many papers are not explicit about the features used and the ones omitted and the justification behind. Neither is its effect on the clustering output considered. Besides, pre-processing methods and choice of metric are usually not explicitly mentioned and justified either. These issues all make reproducibility much more challenging.

\paragraph{Ignored Impact Considerations.} For an algorithmic fairness application in clustering, a thorough empirical evaluation would involve hand picking a dataset where some form of bias was applied or an unequal fair treatment was clearly recorded in the clustering output and then applying a fair clustering algorithm to show an improvement in welfare or a reduction in unfairness. This is not easy to do, especially using UCI datasets as some of them are around two decades old \cite{ding2021retiring}. The work of \citet{abbasi2023measuring} shows an interesting and detailed example where methods from fair clustering have been used to mitigate vote access disparities in real life.
However, the vast majority of the literature has not demonstrated such a thorough and clear application of fair clustering. The lack of demonstrated practical applications in the literature is certainly a weakness. 
Moreover, there have been applications of fair clustering to datasets that are arguably not suitable in terms of their impacts on individuals. For example, both \citet{chierichetti2017fair} and \citet{backurs2019scalable} use the UCI diabetes dataset \footnote{https://archive.ics.uci.edu/ml/datasets/diabetes+130-us+hospitals+for+years+1999-2008} to run fair clustering algorithms for the \GF{} notion which would guarantee proportional representation for each group in the cluster. However, given that the diabetes dataset is concerned with a medical application one can argue that the possible heterogeneity that would be present among individuals belonging to different groups such as race or gender are informative and therefore a fair clustering (especially one like \GF{}) is not suited here and may lead the decision maker to reach incorrect conclusions or miss some critical observations of heterogeneous impacts/behaviors that are known to exist among different groups in medical applications \cite{gale2001diabetes,legato2006gender,cheng2019prevalence,spanakis2013race}.

\section{Miscellaneous Issues of Algorithmic Fairness}\label{sec:misc}
In this section we discuss  additional issues in fair clustering which are in large part shared with the broader algorithmic fairness literature but we add context and considerations that are specific to fair clustering.   

\paragraph{The Many Constraints in Fair Clustering and How to Reconcile Them.} At the current moment the literature has produced at least seven different notions of fairness in clustering \cite{dickerson2023doubly,fc_tutorial}. Moreover, each notion that was introduced (while being well-justified in terms of fairness considerations) does not refer to or consider the interaction with the previously introduced fairness notions in clustering. In supervised learning, the work of \citet{kleinberg2016inherent} showed that two desired fairness notions (calibration and balance) cannot be satisfied simultaneously but the fair clustering literature has not considered the interaction of the different fairness notions with the exception of the recent work of \citet{dickerson2023doubly} and \citet{kellerhals2023proportional}. Specifically, \citet{dickerson2023doubly} show that \GF{} and another group fairness constraint\footnote{This other fairness constraint is diversity in center selection (\DS{}). In the \DS{} constraint, centers are selected from the given set of points which belong to different groups and each group must have a pre-specified number of centers to ensure group diversity in the selected centers.} that was considered in \citet{kleindessner2019fair} and \citet{hotegni2023approximation} can be satisfied simultaneously despite the fact that each of them is incompatible (having an empty feasible set) with a number of distance-based fair clustering notions\footnote{A distance-based fair clustering notion is one that uses the distance between the points in the definition of the fairness notion. Both \GF{} and \DS{} are not distance-based whereas the \EQ{} constraint from Subsection \ref{subsec:review_of_fc} is distance-based.} \cite{chen2019proportionally,makarychev2021approximation,ghadiri2021socially,abbasi2020fair,jung2019center}. In a similar direction, \citet{kellerhals2023proportional} show that any approximation algorithms for the individual fairness notion of \citet{jung2019center} approximates the proportional fairness notion of \citet{chen2019proportionally} and vice versa. This still leaves a number of open questions: Are there other fairness notions in clustering that are also compatible with one another? Are there more general notions which possibly encompass existing ones? More importantly, is this approach of introducing different constraints and satisfying them scalable? How does one build an algorithm which satisfies or makes a trade-off between numerous different notions? Even if one was to forgo algorithms with theoretical guarantees\footnote{Note that almost all of the papers in fair clustering introduce algorithms with theoretical guarantees on the clustering objective and the bound on the fairness violations.} and use heuristics instead the large number of notions to consider would make such heuristics highly non-trivial to design.


\paragraph{Explainable Algorithms} Explainability has become an important consideration in machine learning, especially in applications that have societal and user-welfare considerations \cite{rudin2019stop,arrieta2020explainable}. In clustering there has been recent work on explainable algorithms that can give users a simpler interpretation of the final clustering output \cite{frost2020exkmc,dasgupta2020explainable,makarychev2021near}. One can naturally see that it is desirable to have algorithms that are both fair and explainable since both are important considerations when the welfare of individuals are at stake, but we are not aware of any paper that combines both fairness and explainability in clustering.

\paragraph{Robustness to Strategic Manipulations.} It is not unexpected for individuals to misreport their information or adapt their behaviour according to the deployed algorithm to achieve the best outcome \cite{chen2018strategyproof,bechavod2022information,hardt2016strategic}. Yet we have not so far seen strategic considerations in the fair clustering literature although they are well-motivated. For example, in the \ML{} setting individuals can introduce ``strategic'' noise to their feature vector or misreport their address in an \OR{} setting to be assigned to better centers/facilities and therefore receive better outcomes. 

\paragraph{Satisfying Group Fairness Notions When Group Memberships Are Not Known.} The vast majority of group fairness algorithms in various settings assume knowledge of the group memberships. Yet in many practical applications group memberships are imperfectly known or even completely unknown. While the fair classification literature has paid significant attention to this problem \cite{taskesen2020distributionally,wang2020robust,prost2021measuring,awasthi2020equalized,hashimoto2018fairness} --with the exception of \citet{esmaeili2020probabilistic} which considers the \GF{} constraint-- the problem has remained largely ignored in fair clustering. One should also note that while the work of \citet{esmaeili2020probabilistic} has considered this problem, it makes the strong assumption of having complete probabilistic knowledge of the group memberships and the weak guarantee of satisfying the fairness constraints in expectation\footnote{In a given cluster $i$ and color $h$, the \GF{} constraint is satisfied in \citet{esmaeili2020probabilistic} according to the expected number of points of color $h$ in cluster $i$. Since it is assumed that we have probabilistic color assignments for each point this expectation can be calculated.} not deterministically. Therefore, effective algorithms for this salient problem are needed.

\section{A Path Towards Impactful Fair Clustering Research}\label{sec:towards_more_impactful}

Based on the shortcomings and issues pointed out in the previous sections, in this section we highlight a collection of directions that could lead to more impactful fair clustering research, along with the potential challenges they include. A thoughtful reader might find these ideas intertwined at times. 


\paragraph{Concrete Applications and Representative Real World Data.} The issue of lacking concrete applications is mentioned in Section \ref{subsec:ml_shortcomings} (briefly highlighted in Subsection \ref{subsec:ambig}) and Section \ref{sec:datasets_practical}. We believe that more work along the lines of \citet{abbasi2023measuring} which gives a concrete application of fair clustering would bring great benefits. With concrete applications in mind, algorithm designers could model utility and welfare thus allowing fair clustering to overcome the salient shortcomings mentioned in Section \ref{sec:full_utility}. Further, since such applications are likely to reveal deficiencies in the existing fairness notions this would represent an opportunity to improve the existing notions and introduce more effective ones. 

As mentioned in Section \ref{sec:datasets_practical} the literature has not yet produced a dataset that is truly representative of the fair clustering problem. 
Given the above point of having concrete applications, multiple datasets would be needed to capture different variants of the fair clustering problem. Such datasets intended for fair clustering tasks should also come with their datasheets, i.e. including details such as motivation, composition, collection processes and recommended uses as suggested by \citet{gebru2021datasheets}. Such descriptions would give fair clustering algorithm designers clarity as to what datasets to test their algorithm on. It is critical that such datasets come from real life settings and reflect realistic distributional information. 
Furthermore, theoretical results such as incompatibility between fair clustering notions and the unboundedness of the price of fairness as shown in \citet{dickerson2023doubly} and \citet{esmaeili2021fair} could be too pessimistic as they are based on worst case analysis. Having representative datasets with realistic distributions would enable us to better gauge the level of incompatibility and the ``true'' PoF of fairness notions in real life instances.  

Building real world datasets in fair clustering settings could be challenging, especially in potentially high stakes applications. Similar challenges were mentioned by \citet{patro2022fair} in fair ranking, such as the legal obligations of following privacy and data minimization policies. However, ranking is often used in recommender systems that tend to be data-rich and more able to obtain sensitive information. On the other hand, many of the applications of fair clustering (especially in \OR{} settings) the collection of sensitive information might be restricted. Further, in \OR{} settings public entities such as schools or hospitals could collect and aggregate these datasets, but these entities would still need to go through privacy processing methods. A difficulty that is unique in such cases could be that the datasets are of a much smaller size leading methods such as differential privacy to be less effective\footnote{As a rule, differential privacy degrades the output of an algorithm but this degradation diminishes as the dataset size increases \cite{dwork2014algorithmic}.}.


\paragraph{Taking the Utility and Welfare Effects Into Account More Rigorously and Clearly.} 
As noted in Section \ref{sec:full_utility}, unlike in supervised-learning settings where significant progress was made in characterizing welfare \cite{heidari2019long,hu2020fair,mladenov2020optimizing,chen2021fairness}, the fair clustering literature has been lacking in terms of full welfare characterizations. Introducing such a welfare-centric optimization approach would be impactful and possibly offer a simpler alternative to the existing approach that has resulted in over seven different constraints. Even if a welfare-centric optimization approach is not fully realized, having an approximate picture of the utility would help avoid causing possible harms that come from using a fair clustering algorithm that is mostly focused on a restricted consideration. 

\paragraph{Long-term Fair Clustering} When deploying a fair algorithm it is important to consider its long-term effects and avoid a myopic perspective that only considers its immediate outcomes. Algorithms interact with the environment and as a result change the environment they were intended to operate on \cite{perdomo2020performative}. In interactive learning paradigms some existing models already study temporal aspects and optimize for long-term fairness \cite{d2020fairness,yin2023long,liu2018delayed,liu2020disparate,hu2020fair}. Unfortunately the formulation and evaluation of fairness in clustering problems in its current state is more static. Accurate modeling and evaluations of long-term effects of fair algorithms should rely on works from social and behavioral research for realistic feedback and must be application specific. Within fair clustering, applications in \OR{} settings (hospital and school selection) clearly have long-term effects and fairness concerns, whereas in \ML{} settings the picture is more blurred and long-term effects again depend on downstream tasks performed. \citet{patro2022fair} provides an extensive list of simulation frameworks that already exist in interactive learning settings and that can be adapted to examine long-term fairness effects. In its current state, there is no existing framework that has such a capability in clustering. A framework that implements various fair clustering algorithms along with their dynamic interaction with the environment would be a first step in this direction.

\paragraph{A Framework for Using Fair Clustering and Standards.} Given that we have many different notions in clustering, a perplexing question is what general framework and fairness notion is best to use? Further, when should we choose group fairness over individual fairness? When are distance-based fairness notions preferred? These questions demand standards in applying fair clustering algorithms and answering them could have significant impact. 



\paragraph{Engaging Stakeholders.} Like other algorithmic fairness settings it is important to engage the stakeholders, i.e. the individuals and communities that will be affected by fair clustering algorithms. For example the work of \citet{saha2020measuring} carries a study to investigate the opinions of average individuals about some fairness notions in machine learning and whether they even understand them. Similarly, \citet{yaghini2021human} conducts surveys to understand people's fairness assessments which are then utilized to design a fairness notion. In fair clustering, while some notions like \GF{} are simple and motivated by the established doctrine of disparate impact, notions like \EQ{} are more elaborate and could lead to odd behaviour as shown in Figure \ref{fig:welfare_ex_equitable}. Accordingly, it is not clear whether individuals who could be affected by such notions would even understand or agree with them. Inputs and feedback from stakeholders can help us improve on existing notions and introduce more interpretable and realistic ones.

\bibliography{refs}
\bibliographystyle{plainnat}

\end{document}